\pdfoutput=1

\documentclass[twoside,accepted]{article}
\usepackage{aistats2e}
\usepackage[scaled]{helvet}
\usepackage[british]{babel}
\usepackage{graphicx}
\usepackage{amsthm,amsmath,amsfonts,latexsym}
\usepackage{subfig}
\usepackage{float}
\usepackage{algorithm,algorithmicx,algpseudocode}
\usepackage{mathtools,commath}
\usepackage[all]{xy}
\usepackage{booktabs}
\usepackage{wdel}
\usepackage{soul} %
\usepackage{stmaryrd} %
\usepackage{enumerate}
\usepackage{nicefrac}
\usepackage{natbib}
\usepackage{natbibspacing}
\newcommand{\transposed}{\text{T}}

\newcommand{\expname}[1]{\textsc{#1}}
\newcommand{\ex}[1]{*+[o][F]{\expname{\small#1}}}
\newcommand{\identity}{\mathbf I}
\newcommand{\onesvec}{\mathbf 1}

\newcommand{\support}{\operatorname{Sp}}

\floatstyle{plain}
\restylefloat{figure}
\restylefloat{table}
\restylefloat{algorithm}
\newtheorem{theorem}{Theorem}

\theoremstyle{definition}
\newtheorem{definition}{Definition}
 \newtheorem{exampleX}{Example}[section] %

\newenvironment{example}[1][nope]{%
  \ifthenelse{\equal{#1}{nope}}{\begin{exampleX}}{\begin{exampleX}[#1]}%
  \pushQED{\qed}%
}{%
  \popQED%
  \end{exampleX}%
}

\let\del\undefined
\let\set\undefined
\let\sbr\undefined

\DeclarePairedDelimiter\del{\lparen}{\rparen}
\DeclareTripledDelimiter\delc{\lparen}{\vert}{\rparen}{\mathord}
\DeclareTripledDelimiter\delcc{\lparen}{\Vert}{\rparen}{\mathord}
\DeclarePairedDelimiter\set{\lbrace}{\rbrace}
\DeclareTripledDelimiter\setc{\lbrace}{\vert}{\rbrace}{\mathrel}
\DeclarePairedDelimiter\sbr{\lbrack}{\rbrack}
\DeclareTripledDelimiter\sbrc{\lbrack}{\vert}{\rbrack}{\mathord}
\DeclarePairedDelimiter\tuple{\langle}{\rangle}
\DeclarePairedDelimiter\card{\lvert}{\rvert}

\DeclareMathSymbol{:}{\mathord}{operators}{`:}
\newcommand{\df}{\vcentcolon\nolinebreak\mkern-1.2mu=}

\newcommand{\algoref}[1]{Algorithm~\ref{#1}}

\newcommand{\tabref}[1]{Table~\ref{#1}}
\newcommand{\exaref}[1]{Example~\ref{#1}}
\newcommand{\markdef}[1]{\emph{#1}} %
\newcommand{\family}[3][-1]{\tuple[#1]{#2}_{#3}}  %

\newcommand{\sups}[1][nope]{%
\ifthenelse{\equal{#1}{nope}}{}{^{#1}}}
\newcommand{\subs}[1][nope]{%
\ifthenelse{\equal{#1}{nope}}{}{_{#1}}}

\newcommand{\notationbuddy}[2]{%
  \uppercase{\def\thisupper{#1}}%
  \expandafter\newcommand\csname\thisupper\endcsname{\mathcal{\uppercase{#2}}}
  \expandafter\newcommand\csname#1\endcsname{#2}
  \expandafter\newcommand\csname#1s\endcsname[1][1:T]{{\csname#1\endcsname}_{##1}}
  \expandafter\newcommand\csname#1rv\endcsname{\uppercase{#2}}
  \expandafter\newcommand\csname\thisupper s\endcsname[1][1:T]{{\csname#1rv\endcsname}_{##1}}
}

\notationbuddy{x}{x} %
\notationbuddy{a}{a} %
\notationbuddy{e}{e} %
\notationbuddy{q}{q} %

\renewcommand{\xi}{\e}         %
\renewcommand{\Xi}{\E}         %
\newcommand{\Xis}{\Es}         %
\let\actions\A                 %
\let\expert\e                  %
\let\Expert\erv                %
\newcommand{\h}{\A}
\renewcommand{\H}{\mathcal H}

\newcommand{\N}{\mathbb N}                        %
\let\reals\R

\DeclareMathOperator{\loss}{\ell}                    %

\DeclareMathOperator{\Spmass}{p}                  %
\DeclareMathOperator{\ent}{H}                     %

\newcommand{\pmass}{{\Spmass}\subs}

\newcommand{\hmml}[1]{\operatorname{#1}}             %

\newcommand{\init}[1]{\hmml{#1}_\circ\sups}          %
\newcommand{\tf}[1]{\hmml{#1}_\shortrightarrow\sups} %
\newcommand{\pf}[1]{\hmml{#1}_\shortdownarrow\sups}  %

\newcommand{\pinit}{\init{p}}              %
\newcommand{\ptf}{\tf{p}}                  %
\newcommand{\ppf}{\pf{p}}                  %

\newcommand{\ntransitions}{h}             %
\newcommand{\nproductions}{g}             %

\newcommand{\hmm}[1]{\mathfrak{#1}}
\renewcommand{\A}{\hmm{H}}                          %

\newcommand{\modifier}[1]{\textnormal{\sffamily #1}}
\renewcommand{\sl}{\modifier{sl}}
\newcommand{\fr}{\modifier{fr}}
\newcommand{\fos}{\modifier{v}}%

\newcommand{\Part}{\mathbb C}                     %
\newcommand{\Cell}{\mathcal C}                    %

\newcommand{\PredPost}{\lambda}      %
\newcommand{\Pred}{\pmass\sups}
\newcommand{\ExPred}[1]{\Pred[#1]}             %
\newcommand{\ExPreds}{\Pred[\Xi]}              %

\newcommand{\fs}{\modifier{FS}}
\newcommand{\bayes}{\modifier{B}}
\newcommand{\dm}{\modifier{DM}}

\newcommand{\rv}[1]{*+=+[o][F]{#1}}     %
\newcommand{\lt}[2]{#1^*{#2}}           %

\newcommand{\exclude}[1]{}

\makeatletter
\let\@copyrightspace\relax
\makeatother

\begin{document}
\twocolumn[

\aistatstitle{Switching between Hidden Markov Models using Fixed Share}

\aistatsauthor{Wouter M. Koolen \and Tim van Erven}

\aistatsaddress{Centrum Wiskunde \& Informatica (CWI)\\
P.O. Box 94079, NL-1090 GB Amsterdam, The Netherlands\\
\textsc{\{Wouter.Koolen, Tim.van.Erven\}@cwi.nl}} 
]

\abovedisplayskip 8.5pt plus3pt minus4pt
\belowdisplayskip \abovedisplayskip

\exclude{
\pagenumbering{roman}  %
\tableofcontents
\clearpage
\pagenumbering{arabic} %
}

\begin{abstract}
  In prediction with expert advice the goal is to design online
  prediction algorithms that achieve small regret (additional loss on
  the whole data) compared to a reference scheme. In the simplest such
  scheme one compares to the loss of the best expert in hindsight. A
  more ambitious goal is to split the data into segments and compare to
  the best expert on each segment. This is appropriate if the nature of
  the data changes between segments. The standard fixed-share algorithm
  is fast and achieves small regret compared to this scheme.

  Fixed share treats the experts as black boxes: there are no
  assumptions about how they generate their predictions. But if the
  experts are learning, the following question arises: should the
  experts learn from all data or only from data in their own segment?
  The original algorithm naturally addresses the first case. Here we
  consider the second option, which is more appropriate exactly when the
  nature of the data changes between segments. In general extending
  fixed share to this second case will slow it down by a factor of $T$
  on $T$ outcomes. We show, however, that no such slowdown is necessary
  if the experts are hidden Markov models.
\end{abstract}

\section{Introduction}

In \emph{prediction with expert advice}~\citep{cesa-bianchi2006} a
sequence of outcomes $x_1,x_2,\ldots$ needs to be predicted, one outcome
at a time. Thus, prediction proceeds in rounds: in each round we first
consult a set of experts, who give us their predictions. (We use the
word \emph{expert} for any source of predictions that is available to us
as input.) Then we make our own prediction and incur some loss based on
the discrepancy between our prediction and the actual outcome.
Predictions may for example be in the form of a probability distribution
on outcomes. Loss may be logarithmic loss, i.e.\ the negative logarithm
of the probability assigned to the outcome that actually occurs. The
goal is to minimise our \emph{regret}, which is the difference between
our own cumulative loss on the whole data and the cumulative loss of a
\emph{reference scheme}, which typically involves tuned parameter settings unknown to us
when we make our predictions. For the reference scheme there are several
options; we may, for example, compare ourselves to the cumulative loss
of the best expert in hindsight (after observing the data). A more
ambitious scheme, called \emph{tracking the best expert}, is addressed
by the fixed-share algorithm of \citet{HerbsterWarmuth1998}.

\subsection{Tracking the Best Expert}

In tracking the best expert (TBE), the goal is to achieve small regret
compared to the following reference scheme: 
\begin{enumerate}[(a)]
\setlength{\itemsep}{0pt}
\setlength{\parskip}{0pt}
\item\label{it:partition.data}
  Split the data into segments.
\item \label{it:choose.experts}
  Select an expert for each segment.
\item \label{it:sum.expert.loss}
  Sum the loss of the selected experts on their segments.
\end{enumerate}
This reference scheme is appropriate if the nature of the data changes
between segments. It is harder than comparing to the single best expert
in hindsight, because now there are more unknowns: both the segmentation
(step \ref{it:partition.data}) and the reference experts (step
\ref{it:choose.experts}) are unknown when we make our predictions. In
particular the reference experts may be the best experts in hindsight
for their assigned segments.

In 1995 Herbster and Warmuth introduced an efficient algorithm called
\emph{fixed share} (\fs) and showed that it achieves small regret (see
\thmref{thm:fixed.share.loss.bound} below) compared to the TBE reference
scheme of \citet{HerbsterWarmuth1998}. Given the predictions of the
experts, the algorithm's running time is linear in the number of
outcomes and linear in the number of experts. Problem solved. Or is it?

\subsection{Learning Experts}

In this paper we take another look at the TBE reference scheme for
\emph{learning experts} and ask: if an expert is selected for some
segment, then should the expert learn from all data or only from the
data in that segment?

We may assume that the experts do not know the segmentation chosen in
step \ref{it:partition.data} of the reference scheme. (Otherwise, why
wouldn't we just ask them?) Hence if we treat the experts as black boxes
and only ask for their prediction at each time step as
in~\citep{HerbsterWarmuth1998}, it is natural that they learn from all
data. We call this the \emph{standard} interpretation of the TBE
reference scheme (S-TBE).

However, as the following example will illustrate, it may be beneficial
if experts learn only from the segment for which they are selected,
because they may get confused by data in other segments that follow a
different pattern. We call this the \emph{local learners} interpretation
of tracking the best expert (LL-TBE). As a slight complication, it will
turn out that in LL-TBE we have a further choice: whether to tell a
learning expert the timing of its segment or not, which generally makes
a difference. When segment timing is preserved, we call the resulting
reference scheme \emph{sleeping LL-TBE}; when segment timing is
\emph{not} preserved we call the reference scheme \emph{freezing
LL-TBE}. The next example demonstrates that S-TBE and the two variants of
LL-TBE are really different reference schemes.

\paragraph{Example: Drifting Mean}

In applications one would usually build up complicated prediction
strategies from simpler ones in a hierarchical fashion. For example, let
us first define simple static experts, parametrised by $\mu \in \reals$,
which predict according to a standard normal distribution with mean
$\mu$ in each round. Now define a learning expert $\dm[\theta]$ that has
a stochastic model for the (unobservable) drift of $\mu$ over time. This
\emph{drifting mean} learning expert predicts according to a hidden
Markov model in which the hidden state at time $t$ is $\mu_t$ and the
production probability of an outcome given $\mu_t$ is determined by the
simple expert with parameter $\mu_t$. Initially, $\mu_1 = 0$ with
probability one. Then $\mu_{t+1} = \mu_t+1$ with probability $\theta$
and $\mu_{t+1} = \mu_t$ with probability $1-\theta$ for some fixed
parameter $\theta$. (See \figref{fig:dontlookback}.)

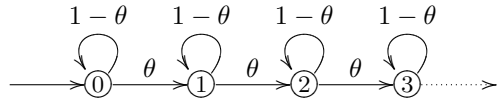
\begin{figure}
\centering
$\xymatrix@R=0.5em@C=2.8em{
\ar[r] &
\ex{0} \ar[r]^*{\theta} \ar@(ur,ul)_*{1-\theta} & 
\ex{1} \ar[r]^*{\theta} \ar@(ur,ul)_*{1-\theta} & 
\ex{2} \ar[r]^*{\theta} \ar@(ur,ul)_*{1-\theta} & 
\ex{3} \ar@{.>}[r] \ar@(ur,ul)_*{1-\theta} & 
}$
\caption{State Transitions for Learning Expert $\dm[\theta]$, which learns a drifting mean}\label{fig:dontlookback}
\end{figure}

The expert $\dm[\theta]$ may be said to be learning, because its
posterior distribution of $\mu_t$ given outcomes $x_1,\ldots,x_{t-1}$
indicates how much credibility the expert assigns to each value of
$\mu_t$: high weight on, say, $\mu_t = 3$ indicates that $\dm[\theta]$
considers it likely for $\mu_t = 3$ to give the best prediction for
$x_t$.

\begin{figure}
\centering
\subfloat[Suitable Sleeping LL-TBE Data\label{fig:sub.sl.data}]{\includegraphics[width=.24\textwidth]{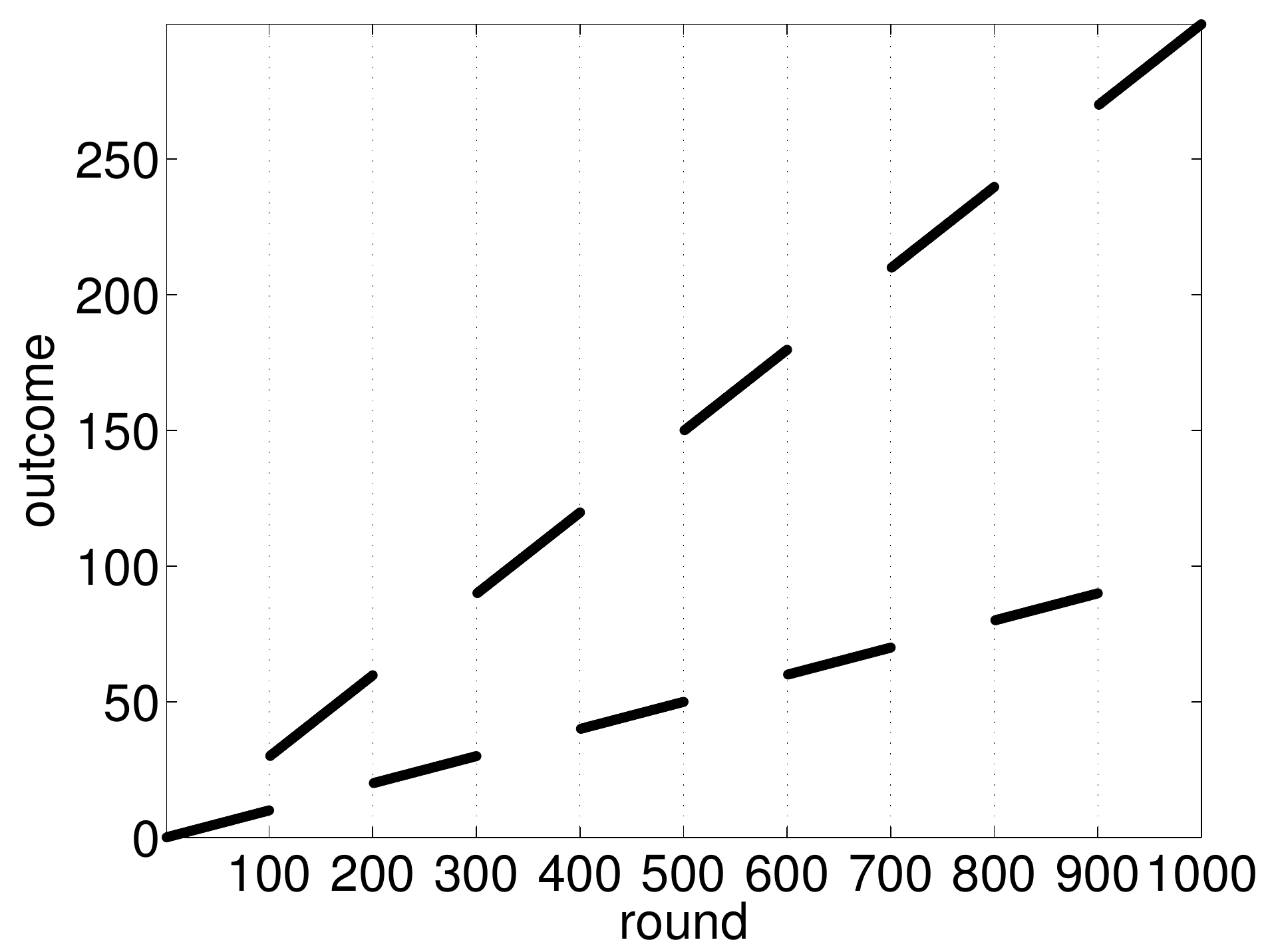}}~
\subfloat[Suitable Freezing LL-TBE Data\label{fig:sub.fr.data}]{\includegraphics[width=.24\textwidth]{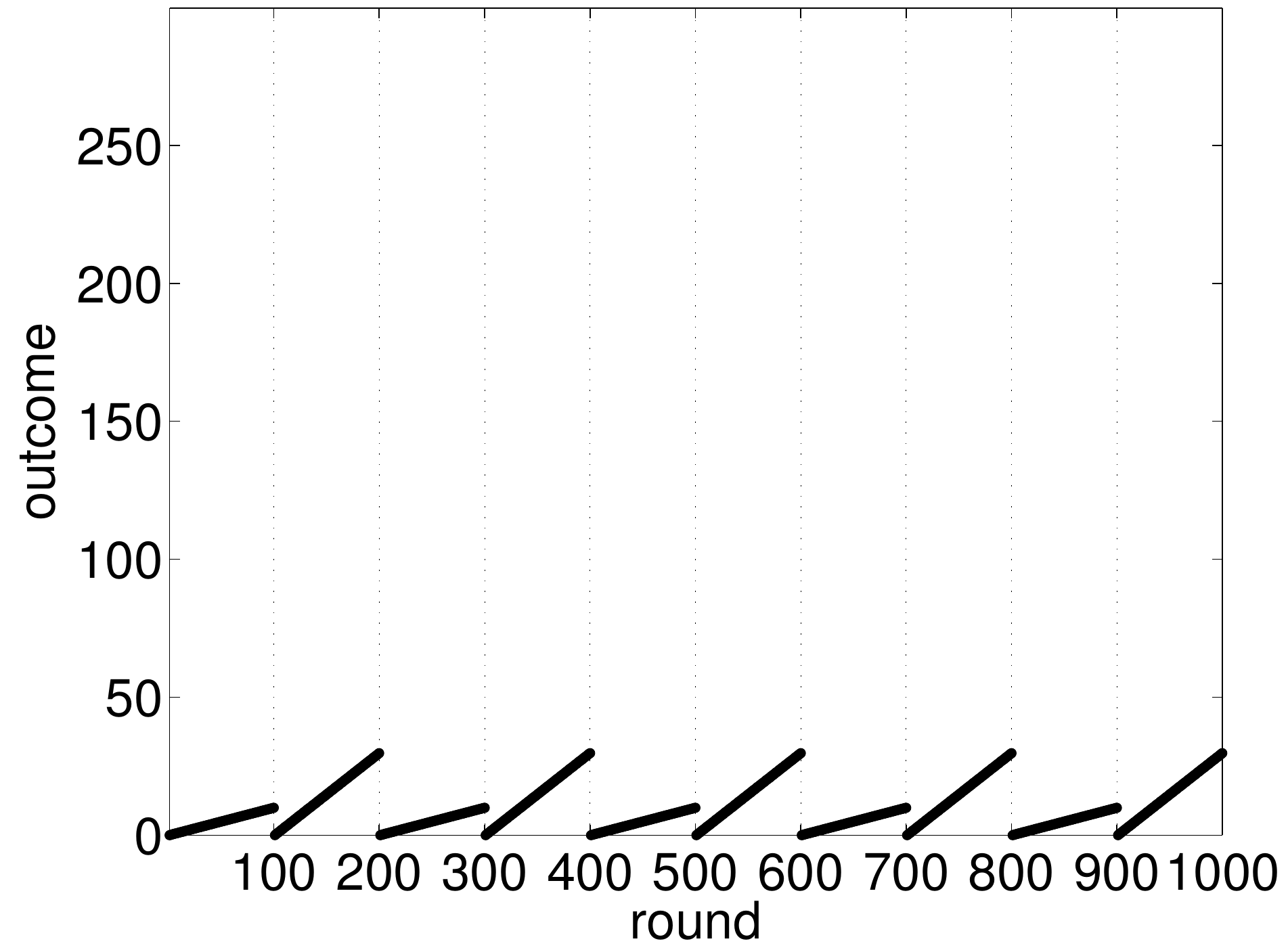}}
\\
\subfloat[Cumul.\ Loss on Data \subref{fig:sub.sl.data}\label{fig:sub.sl.loss}]{\includegraphics[width=.24\textwidth]{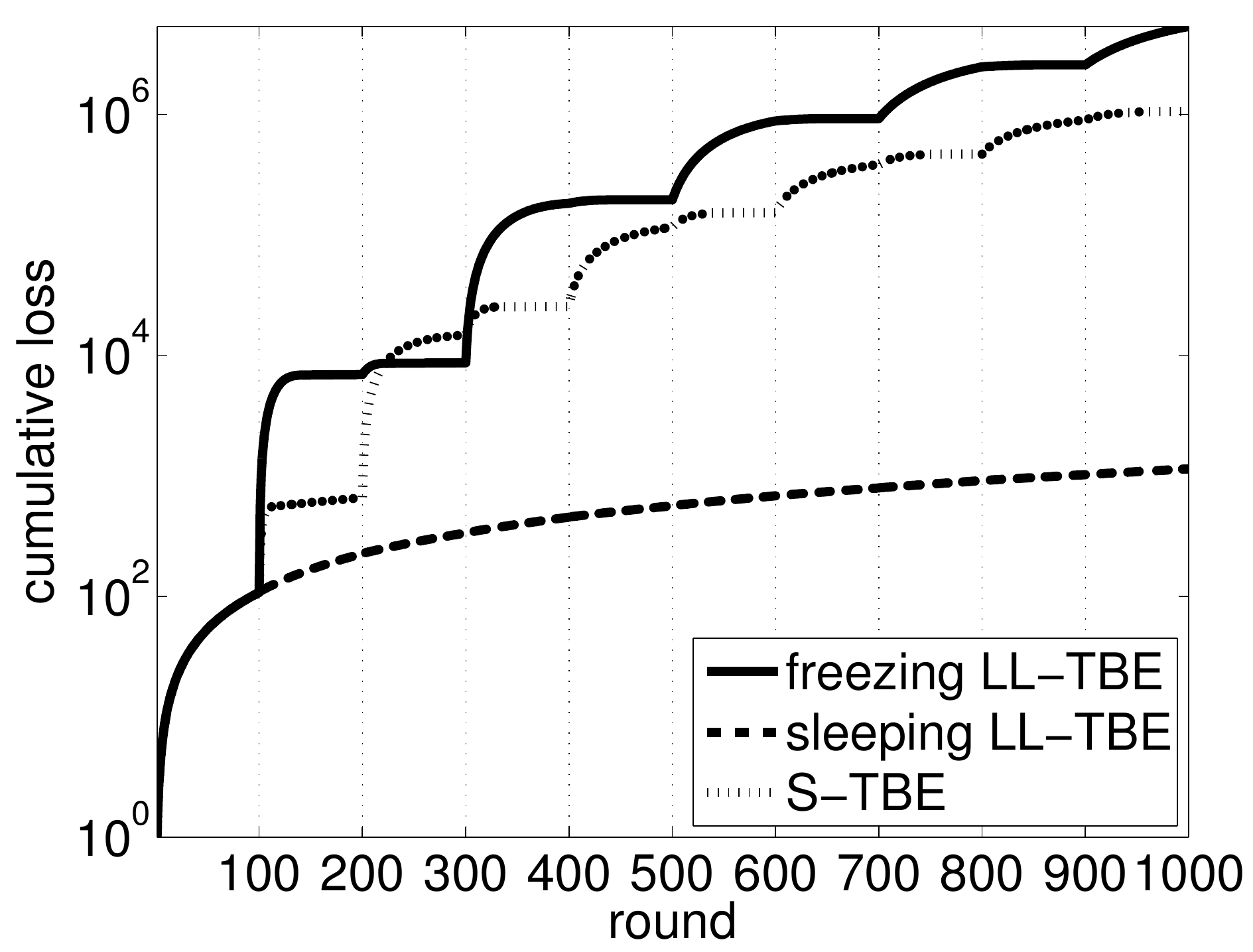}}~
\subfloat[Cumul.\ Loss on Data \subref{fig:sub.fr.data}\label{fig:sub.fr.loss}]{\includegraphics[width=.24\textwidth]{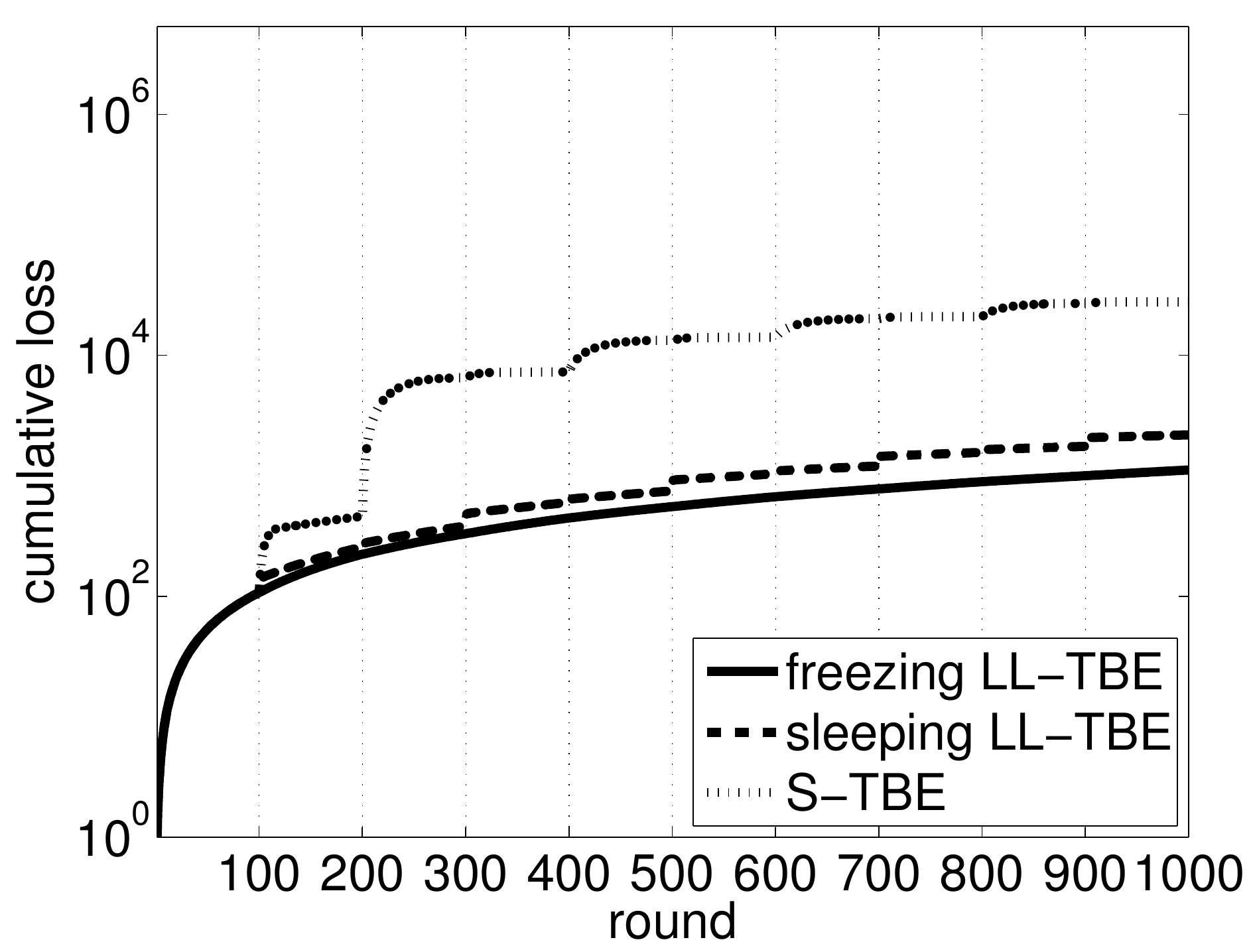}}
\caption{The Difference Between S-TBE and the Two LL-TBE Reference
Schemes. Note the logarithmic scale of the y-axis in
\subref{fig:sub.sl.loss} and \subref{fig:sub.fr.loss}!}
\end{figure}

Figures~\ref{fig:sub.sl.data} and \ref{fig:sub.fr.data} plot two
artificial data sets. For \figref{fig:sub.sl.data} sleeping LL-TBE is
appropriate, for \figref{fig:sub.fr.data} freezing LL-TBE is more
suitable. The data consist of $10$ segments of $100$ outcomes. In each
segment the outcomes are increasing deterministically at a rate of either
$0.1$ or $0.3$ per outcome. Note that for the freezing data all segments
start from $0$, whereas for sleeping any segment looks like the proces
that generated it started at $0$ at time $1$, but went unobserved for a
while.

Figures~\ref{fig:sub.sl.loss} and \ref{fig:sub.fr.loss} show the
cumulative log(arithmic) loss for all three TBE reference schemes. Note
that the difference between the schemes is so large that their losses
had to be plotted on a logarithmic scale. In each case we consider two
experts: $\dm[0.1]$ and $\dm[0.3]$ and use the expert $\dm[\theta]$ for
any segment with rate $\theta$. The difference between the three schemes
lies in which data is used by $\dm[\theta]$ to learn from. In the S-TBE
scheme $\dm[\theta]$ is shown all the data, even those outside the
segment it has to predict. In the two LL-TBE schemes, on the other hand,
a fresh copy of $\dm[\theta]$ only sees the data in the segment for
which it is selected: for freezing LL-TBE, $\dm[\theta]$ predicts as if
the current segment is the only data; for sleeping LL-TBE, $\dm[\theta]$
knows the timing of the segment it is predicting, and treats all samples
preceding that segment as unobserved. Thus in sleeping LL-TBE the original
timing of the segments is preserved, while in freezing LL-TBE it is
lost.

We see that for the sleeping data the sleeping LL-TBE reference scheme
has much smaller loss than the other two schemes. And for the freezing
data the freezing LL-TBE scheme has the smallest loss by far. (Mind the
logarithmic scale of the y-axis, which puts the loss of sleeping LL-TBE
deceptively close to the loss of freezing LL-TBE in
\figref{fig:sub.fr.loss}: a constant offset indicates a fixed multiplicative overhead.) In both
cases the reason for the large differences between the reference schemes
is that $\dm[\theta]$ gets confused if it learns from the wrong data.

\subsection{Expert Hidden Markov Models}

The learning expert $\dm[\theta]$ in the example above is a hidden Markov
model in which the production probabilities (of outcomes given the state) depend on lower-level base
experts. In general such prediction strategies are called \emph{expert
hidden Markov models} (EHMMs). The use of EHMMs is not
restricted to describing learning experts. For example, many algorithms
for prediction with expert advice, including $\fs$
itself, can be represented as EHMMs (see Koolen and De Rooij
\citeyearpar{koolen08:_combin_exper_advic_effic} and its references, and
\citet{Jaakkola2003}). In addition any ordinary HMM is trivially an EHMM:
just introduce lower-level base experts for its production
probabilities. Not every algorithm can be represented as an EHMM,
however. The follow-the-perturbed-leader algorithm by \citet{hannan57} and
\emph{variable share} by \citet{HerbsterWarmuth1998}, for instance, are
exceptions.

\subsection{Fixed Share for Learning
Experts}\label{sec:fixedsharelearningexperts}

\paragraph{LL-TBE Requires More Information}

The example above shows that there is a large difference between S-TBE and
the sleeping or freezing LL-TBE reference schemes. One may therefore
wonder whether there exists an algorithm that achieves small regret
compared to LL-TBE. Unfortunately, no algorithm will be able to do the
job without additional knowledge about the learning experts. To see
this, note that the reference scheme may split the data into segments in
any way it sees fit. But black-box experts are not telling us what their
predictions would be for any possible segmentation; they only give us a
single prediction each round. Therefore, even if we knew the
segmentation and the selected expert for each segment, we still would
have insufficient information to achieve the reference scheme. The only
way to address this problem is to get more information about the
learning experts. This information should have an efficient
representation and should somehow tell us what the learning experts
would predict for any possible segmentation.

\paragraph{Copying Experts is Less Efficient}

The straight-forward approach would be to introduce a fresh copy of each
expert for each possible start of a new segment and run the original
fixed-share algorithm on the resulting enriched set of experts. But then
the number of experts would grow linearly with the number of rounds, and
consequently the total running time would go up from linear to quadratic
in the number of outcomes. As this makes the difference between an
online algorithm that can run forever and an algorithm that effectively
comes to a stop after, say, $10^5$ outcomes, it is worth seeing whether
such an increase in running time is really unavoidable.

\paragraph{EHMMs: the Efficient Special Case}

As we will show, it turns out there is a special class of learning
experts for which no increase in running time is necessary. These are
the learning experts that can be described in EHMM form. Although this
excludes learning experts that for example implement
follow-the-perturbed-leader, the class of EHMMs is still rich enough to
be of interest, if only because it includes all ordinary HMMs.
In the interpretation of the two LL-TBE reference schemes for learning
experts in EHMM form, we do need to be careful if the base experts in
the EHMMs are learning themselves: because we make no assumptions about
the base experts, they always learn from all the data.

\paragraph{Main Result: Achieving LL-TBE Efficiently}

We present two new algorithms: $\fs^\sl$ for sleeping LL-TBE and
$\fs^\fr$ for freezing LL-TBE, which both generalise $\fs$. We show that
these algorithms achieve the same regret bound compared to their
respective LL-TBE reference schemes as $\fs$ achieves compared to the
S-TBE reference scheme. In addition, $\fs^\sl$ runs equally fast as the
original fixed-share algorithm; for
$\fs^\fr$ no slowdown occurs either if the EHMMs for the learning
experts have a finite number of hidden states, otherwise it is typically
still faster than just copying the experts.

Like fixed share, our new algorithms can be represented as EHMMs. In
fact, we will build up both algorithms by describing how to combine the
EHMMs for the learning experts, which the algorithms get as inputs, into
a single larger EHMM. Apart from introducing the LL-TBE reference
scheme, this construction is our main result: regret bounds follow from
the EHMM representations using methods described in \citetext{Koolen and
De Rooij,\, \citeyear{koolen08:_combin_exper_advic_effic}}, and the
algorithms are simply instances of the forward algorithm for EHMMs.

\subsection{Overview}

We start by formally introducing prediction with expert advice in the
next section. Then \secref{sec:ehmm} reviews EHMMs, including the
representation of $\fs$ as an EHMM. It is
shown how the standard regret bound for $\fs$
by~\citet{HerbsterWarmuth1998} can be proved using this representation.
In \secref{sec:TBLE} we formally define the freezing and sleeping
LL-TBE reference schemes and present our new algorithms. Then we prove
their regret bounds and state their running times.

\section{Preliminaries: Prediction With Expert Advice}\label{sec:preliminaries} \label{sec:expertadvice}

In this section we formally introduce the online learning setting of
prediction with expert advice. In this setting prediction proceeds in
rounds. In each round $t$, we first receive advice from each expert $\xi
\in \Xi$ in the form of an action $\as[t]^{\xi} \in \actions$. Then we
distill our own action $\as[t] \in \actions$ from the expert advice.
Finally, the actual outcome $\xs[t] \in \X$ is observed, and everybody
suffers loss as specified by a fixed loss function $\loss \colon
\actions \times \X \to \intcc{0,\infty}$. Thus, the performance of a
sequence of actions $\as = \as[1], \ldots, \as[T]$ on data $\xs =
\xs[1], \ldots, \xs[T]$ is measured by the cumulative loss
$\loss(\as,\xs) = \sum_{t=1}^T \loss(\as[t],\xs[t])$.

We present our results for \emph{log(arithmic) loss} only, which allows
us to draw on familiar concepts from probability theory, like e.g.\
conditional probabilities and hidden Markov models. Their generalisation
to arbitrary \emph{mixable losses} is straight-forward using the methods
of \citet{Vovk1999}.

\paragraph{Log Loss}
\newcommand{\ExPredDist}[1]{\operatorname{P}_{#1}}
For log loss the actions $\actions$ are probability mass (or density)
functions on $\X$ and $\loss(\Pred,x) = -\log \Pred(x)$ for any $\Pred
\in \actions$, where $\log$ denotes the natural logarithm. Notice that
minimising log loss is equivalent to maximising the predicted
probability of outcome $x$. We write $\ExPred{\xi}_t$ for the prediction
of expert $\xi$ at time $t$ and denote the predictions for all experts
jointly by $\ExPreds_t$. Another important property of the log loss is
the \emph{chain rule}: interpreting any prediction $\Pred_t(\xs[t])$ as
the conditional probability $P(\xs[t] | \xs[<t])$ of outcome $\xs[t]$
given all past outcomes $\xs[<t] = x_1, \ldots, x_{t-1}$, we see that
the cumulative log loss of a sequence of predictions
\begin{equation}\label{eqn:introchainrule}
  \sum_{t=1}^T - \log \Pred_t(\xs[t])
  ~=~
  - \log \prod_{t=1}^T P(\xs[t]|\xs[<t])
  ~=~
  - \log P(\xs)
\end{equation}
equals the negative logarithm of the joint $P$-probability of all data
$\xs$. Thus any lower bound on $P(\xs)$ directly implies an upper bound
on the cumulative loss of predictions $\Pred_1,\ldots,\Pred_T$ on data $\xs$.

\paragraph{Segments} For $m\le n$, we abbreviate the \markdef{segment} $\set{m, \ldots,
n}$ to $m:n$. 
For any sequence $y_1, y_2, \ldots$ and any segment
$\Cell = m:n$ we write $y_\Cell$ for the subsequence $y_m, \ldots, y_n$.
For example, $\xs[m:n] = x_m,\ldots,x_n$ and
$\ExPreds_{1:T} = \ExPreds_1, \ldots, \ExPreds_T$. 
If all segments in a family $\Part = \{\Cell_1, \Cell_2, \ldots\}$ are
pairwise disjoint and together cover $1:T$, then we call $\Part$ a
\markdef{segmentation} of $1:T$. 
We denote by $\family{\e_\Cell \in \E}{\Cell \in \Part}$ the labelling that assigns expert $\e_\Cell$ to segment $\Cell$.

\section{Expert Hidden Markov Models}\label{sec:ehmm}

EHMMs were introduced by Koolen and De Rooij \citeyearpar{koolen08:_combin_exper_advic_effic} as a graphical and computational language to specify strategies for prediction with expert advice. EHMM diagrams directly represent the internal structure of the prediction strategy, facilitating the derivation of loss bounds. Moreover, there is a standard algorithm for sequential prediction, the \emph{forward algorithm}, which greatly simplifies derivation of running time bounds.

In this paper, we use EHMMs in two ways. On the input side, we use them to represent the learning experts whose predictions we want to combine. On the output side, we specify our own prediction strategies based on expert advice as EHMMs. 

An \markdef{EHMM} $\A$ is a probability distribution that is
constructed according to the Bayesian network in \figref{fig:ehmm}. 
\begin{figure}
\centering
\footnotesize
$\xymatrix@!0@R=3.0em@C=3.5em{
\lt{\ar[r]}{\pinit} &
\rv{\Qs[1]} \lt{\ar[r]}{\ptf} \lt{\ar[d]}{\ppf} & 
\rv{\Qs[2]} \lt{\ar[r]}{\ptf} \lt{\ar[d]}{\ppf} &
\rv{\Qs[3]} \lt{\ar[r]}{\ptf} \lt{\ar[d]}{\ppf} & 
\rv{\Qs[4]} \lt{\ar[r]}{\ptf} \lt{\ar[d]}{\ppf} & \cdots
\\
&
\rv{\Xis[1]} \lt{\ar[d]}{\ExPreds_1} & 
\rv{\Xis[2]} \lt{\ar[d]}{\ExPreds_2} & 
\rv{\Xis[3]} \lt{\ar[d]}{\ExPreds_3} &
\rv{\Xis[4]} \lt{\ar[d]}{\ExPreds_4} & \cdots
\\
&
\rv{\Xs[1]} & 
\rv{\Xs[2]} & 
\rv{\Xs[3]} &
\rv{\Xs[4]} & \cdots
}
$
\caption{Bayesian Network Specification of an EHMM\label{fig:ehmm}}
\end{figure}
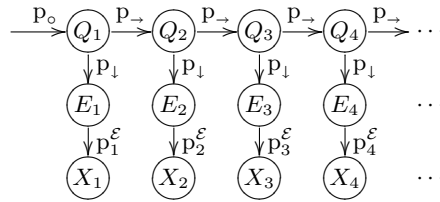
It is used to sequentially predict outcomes $\Xs[1]$, $\Xs[2]$,
$\ldots$, which take values in outcome space $\X$, using advice from a
set of experts $\E$. At each time $t$, the distribution of $\Xs[t]$
depends on a hidden state $\Qs[t]$, which determines mixing weights for
the experts' predictions. Formally, the \markdef{production function}
$\ppf$ determines the interpretation of a state: it maps any state
$\qs[t] \in \Q$ to a distribution $\ppf[{\qs[t]}]$ on the identity
$\Expert_t$ of the expert that should be used to predict $\Xs[t]$. Then
given $\Expert_t = \expert$, the distribution of $\Xs[t]$ is expert
$\expert$'s prediction $\ExPred{\expert}_t$. It remains to define the
distribution of the hidden states. The starting state $\Qs[1]$ has
\markdef{initial distribution} $\pinit$, and the state evolves according
to the \markdef{transition function} $\ptf$, which maps any state
$\qs[t]$ to a distribution $\ptf[{\qs[t]}]$ on its successor states. 

An EHMM $\A$ defines a prediction strategy as follows: after observing
$\xs[<t]$, predict the next outcome $\Xs[t]$ using the marginal
$\A\delc{\Xs[t]}{\xs[<t]}$, which is a \emph{mixture} of the experts'
predictions $\ExPreds_{t}$.

We present four example EHMMs. The first three examples are suitable as
input learning experts, which might be combined in the sleeping or
freezing LL-TBE reference scheme. The fourth example represents
$\fs$ as an EHMM, which will later be helpful
when we compare it to our new generalisations. 

\begin{example}[\figref{fig:dontlookback}: Expert that Learns a Drifting Mean]
Here we formally define the EHMM $\dm[\theta]$ from the example in the
introduction. Recall that the base experts predict according to standard
normal distributions with fixed mean $\mu$, which only takes integer
values. Thus
\begin{equation*}
  \ExPred{\mu}_t(x) \df \frac{1}{\sqrt{2\pi}}\,e^{-(x-\mu)^2/2}
\end{equation*}
for all $\mu \in \E \df \N = \{0,1,2,\ldots\}$. In this EHMM it is
sufficient to have a one-to-one correspondence between hidden states and
experts, such that $Q_t = E_t$. This is expressed by $\Q \df \E$ and
$\ppf \df \identity$, where $\identity$ denotes the identity operator.
The definition of $\dm[\theta]$ is completed by letting the initial
distribution $\pinit$ be a point-mass on $\mu=0$, and defining the
transition function $\ptf$ as in \figref{fig:dontlookback}: for any two
states $\mu,\mu' \in \Q$
\begin{equation*}
  \ptf[\mu](\mu') \df
    \begin{cases}
      \theta & \text{if $\mu' = \mu+1$,}\\
      1-\theta & \text{if $\mu' = \mu$,}\\
      0 & \text{otherwise.}
    \end{cases} \qedhere
\end{equation*}
\end{example}

\begin{example}[Bayes on base experts]\label{ex:bayes.on.experts}
  Consider the Bayesian mixture (also known as the exponentially
  weighted average predictor) of base experts $\E$ with prior $w$. We
  identify this prediction strategy with the following EHMM $\bayes[w]$,
  which makes the same predictions. As in the previous example, let $\Q
  \df \E$ and $\ppf \df \identity$, so that $Q_t = E_t$. This time, however,
  let $\pinit \df w$ and $\ptf \df \identity$. Despite its deceptive
  simplicity, this EHMM \emph{learns}: its marginal distribution of
  $X_{t+1}$ given previous outcomes $x_{1:t}$ is a mixture of the
  base expert's predictions according to the Bayesian posterior. 
\end{example}

\begin{example}[Bayes on EHMMs]\label{exa:bayes.on.ehmms}
  Let $\H = \set{\h^1, \ldots, \h^n}$ be EHMMs with base experts
  $\E^1,\ldots,\E^n$, and let $w$ be a prior on $\H$. Then, instead of
  treating $\h^1,\ldots,\h^n$ as black box predictors as in the previous
  example, their Bayesian mixture can also be expressed as a single EHMM
  $\bayes[w, \H]$ on the union of their base experts $\E \df
  \bigcup_{i=1}^n \E^i$: assume without loss of generality that
  $\h^1,\ldots,\h^n$ have disjoint state spaces $\Q^1,\ldots,\Q^n$ and
  let $\Q \df \bigcup_{i=1}^n \Q^i$. For any state $\q \in \Q^i$, let
  $\ppf[\q]$ equal $\ppf[\q,i]$, where $\ppf[i]$ is the production
  function of $\h^i$, so that all states keep their original
  interpretation. In addition let $\pinit(\q) \df w(i)\pinit^i(\q)$, where
  $\pinit^i$ denotes the initial distribution of $\h^i$. Finally, let
  $\ptf[\q](\q')$ equal $\ptf[\q,i](\q')$, the transition probability from $\q$ to $\q'$
  for $\h^i$ if $\q,\q' \in \Q^i$ and let $\ptf[\q](\q') \df 0$ otherwise.
  Again, this EHMM \emph{learns} which of the EHMMs in $\H$ is the best
  predictor.
\end{example}

\begin{example}[Fixed share]\label{exa:fixed.share}
The fixed-share algorithm take a parameter $\alpha$, called the
\emph{switching rate}. Fixed share with prior distribution $w$
on experts $\E$ and switching rate $\alpha$ can be represented as an
EHMM $\fs[\alpha, w]$ as follows. As in the Bayesian
mixture on base experts, let $\Q \df \E$ and $\ppf \df \identity$, so that
$Q_t = E_t$, and let $\pinit \df w$. Instead of the identity operator,
however, use the transition function
\begin{equation*}
  \ptf \df (1-\alpha) \identity + \alpha w \onesvec^\transposed,
\end{equation*}
where $\onesvec^\transposed$ denotes the operator that sums the
probability masses of all the hidden states. This transition
function may be interpreted as follows: behave like the Bayesian mixture
with probability $1-\alpha$, but with probability $\alpha$ take all the
probability mass and redistribute it according to the prior $w$.
Observe that for any probability distribution $\PredPost$ on states $\Q$, we can compute $\ptf \PredPost = (1-\alpha)\lambda + \alpha w$ in constant time per state.
We also note that in~\citep{HerbsterWarmuth1998} the prior $w$ is always
taken to be the uniform distribution, which gives the best worst-case
regret bound. 
\end{example}

\subsection{Standard Fixed Share Loss Bound}

To demonstrate the graphical derivation of loss bounds for EHMMs we now
prove a regret bound for $\fs$ using its representation as an
EHMM. The general technique is to give lower bounds on the transition
function and the initial distribution. For simplicity the bound we show
is slightly weaker than the standard regret bound~\cite[Corollary
1]{HerbsterWarmuth1998}. (One could get the exact same bound by
taking into account the remark in footnote~3 of~\citetext{Koolen and De
Rooij,\, \citeyear{koolen08:_combin_exper_advic_effic}}, but this
unnecessarily complicates the proof.)

\begin{theorem}\label{thm:fixed.share.loss.bound}
  Fix a prior $w$ on experts $\E$ and a switching rate $\alpha$. Then
  for any data $\xs$, expert predictions $\ExPreds_{1:T}$, reference
  segmentation $\Part$ and assignment of experts to segments
  $\family{\xi_\Cell \in \E}{\Cell \in \Part}$
\newcommand{\braceforimpact}[2]{\underbrace{\vphantom{\sum_{\Cell \in \Part}}#1}_\textnormal{\kern -1cm #2 \kern -1cm}}
\begin{multline*}\label{eqn:fixedsharebound}
  \loss\big(\fs[\alpha,w], \xs\big) ~\le~ \\
    \braceforimpact{\sum_{\Cell \in \Part} \loss(\xi_\Cell, \xs[\Cell])}{S-TBE ref.\ scheme}
      + \braceforimpact{(T-1) \ent(\alpha^*, \alpha)}{Switching}
      + \braceforimpact{\sum_{\Cell \in \Part} -\log w(\es[\Cell])}{Expert selection},
\end{multline*}  
where $\ent(\alpha, \beta) = - \alpha \log \beta - (1-\alpha) \log
(1-\beta)$
and
$\alpha^* = \frac{\card{\Part}-1}{T-1}$.
\end{theorem}

\noindent
Note that if $w$ is the uniform distribution then $-\log w(\es[\Cell]) =
\log |\E|$ for all $\es[\Cell]$. Then the difference with the standard
bound in \citep{HerbsterWarmuth1998} is $(|\Part|-1)(\log |\E| - \log
(|\E|-1))$, which is negligible.

\begin{proof}
Recall that $\fs \equiv \fs[\alpha,w]$ has transition function $\ptf =
(1-\alpha)\identity + \alpha w \onesvec^\transposed$. Therefore for any
reference segmentation $\Part$ the joint probability $\fs(\xs)$ of any data
sequence $\xs$ can be bounded from below by replacing transitions in
$\fs$ \emph{between} segments by $\alpha w \onesvec^\transposed$, and
those \emph{within} the same segment by $(1-\alpha)\identity$. The EHMM
then degenerates into a sequence of independent Bayesian mixture EHMMs
$\bayes[w]$ (see Example~\ref{ex:bayes.on.experts}), one for each
segment. Therefore
\begin{equation*}
  \fs(\xs) \geq \alpha^{\card{\Part}-1}(1-\alpha)^{T-\card{\Part}}
                \prod_{\Cell \in \Part} \bayes[w](\xs[\Cell]).
\end{equation*}
Similarly we can lower-bound the initial distribution of $\bayes[w]$ by
a function that assigns weight $w(\es[\Cell])$ to the expert
$\es[\Cell]$ selected for $\Cell$ in the reference segmentation and is $0$
otherwise. It follows that
  $\bayes[w](\xs[\Cell])
    = \sum_{\es[]} w(\es[]) \ExPred{{\es[]}}_{\Cell}(\xs[\Cell])
    \geq w(\es[\Cell])
    \ExPred{{\es[\Cell]}}_{\Cell}(\xs[\Cell])$,
where $\ExPred{{\es[]}}_{\Cell}(\xs[\Cell])$ denotes the joint
probability of outcomes $\xs[\Cell]$ according to the predictions of
expert $\es[]$.
Hence by \eqref{eqn:introchainrule} we can conclude that
\begin{multline*}
  \loss\big(\fs, \xs\big)
    = -\log \fs(\xs)\notag\\
    \leq -\log \alpha^{\card{\Part}-1}(1-\alpha)^{T-\card{\Part}}
      + \sum_{\Cell \in \Part} -\log \ExPred{{\es[\Cell]}}_{\Cell}(\xs[\Cell])
      -\log w(\es[\Cell])\notag\\
    =(T-1) \ent(\alpha^*, \alpha)
      +\sum_{\Cell \in \Part} \loss(\xi_\Cell, \xs[\Cell])
      + \sum_{\Cell \in \Part} -\log w(\es[\Cell]),
\end{multline*}  
which completes the proof.
\end{proof}

\newcommand{\B}{\hmm{B}}
\newcommand{\C}{\hmm{C}}

\section{Fixed Share for Learning Experts}\label{sec:TBLE}

In this section we define the freezing and sleeping LL-TBE reference
schemes for learning experts. Then, for each scheme, we provide our
prediction strategy $\fs^\fr$ and $\fs^\sl$ and we prove that it
achieves as small regret as $\fs$.

\subsection{LL-TBE and the Loss of an EHMM on a Segment}

In order to state the loss of the freezing and sleeping LL-TBE reference
schemes, we first define the loss of a single learning expert on a
single segment. Then we define the loss of a whole segmentation.

\exclude{%
Consider the learning expert $\dm[\theta]$, which operates on a set of experts $\E$ indexed by the natural numbers, and which learns the index of the expert that is currently best, under the assumption that this index will increase with rate $\theta$ over time. We now consider the predictions of $\dm[\theta]$ on a segment of data $\xs$. Under the freezing interpetation, $\dm[\theta]$ predicts outcomes $\xs[1000:2000]$ (roughly) by quoting experts $0$ to $1000\,\theta$. On the other hand, under the sleeping interpretation $\dm[\theta]$ predicts the same outcomes (roughly) by quoting experts $999\,\theta$ to $1999\,\theta$. We say roughly, since in reality $\dm[\theta]$ generally predicts using a mixture of many experts.
}
Let $\A$ be the EHMM for a learning expert with arbitrary base experts
$\E$. Then the freezing and sleeping probability distributions
$\A^\fr_{i:j}$ and $\A^\sl_{i:j}$ on segment $\xs[i:j]$ are specified by
the Bayesian networks of \figref{fig:sleeping.n.freezing.EHMMs}. For
freezing, the state at time $i$ is simply initialised according to
$\A$'s initial distribution $\pinit$. For sleeping, we forward the
initial distribution to time $i$ by repeatedly applying the transition
function $\ptf$.
Thus, the cumulative freezing and sleeping losses of $\A$ on segment
$\xs[i:j]$ are given by $\loss(\h^\fr_{i:j}, \xs[i:j]) \df - \log
\A^\fr_{i:j}(\xs[i:j])$ and $\loss(\h^\sl_{i:j}, \xs[i:j]) \df - \log
\A^\sl_{i:j}(\xs[i:j])$. Note that we treat the base experts $\E$ as
black boxes, so they may learn from the whole data.

\begin{definition}[LL-TBE reference loss]
Fix data $\xs$ and a set of EHMMs $\H$. Let $\Part$ be a segmentation of $1:T$ and let $\family{\h_\Cell \in \H}{\Cell \in \Part}$ be an assignment of experts to segments. Then the losses of the freezing and sleeping LL-TBE reference schemes are
$\sum_{\Cell \in \Part} \loss(\h^\fr_\Cell, \xs[\Cell])$
 and $\sum_{\Cell \in \Part} \loss(\h^\sl_\Cell, \xs[\Cell])$.
\end{definition}

\noindent
Note that selecting a learning expert on consecutive segments differs from selecting that expert on their union, since experts are reset between segments.

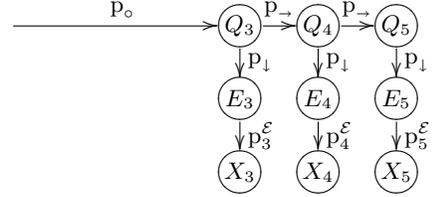
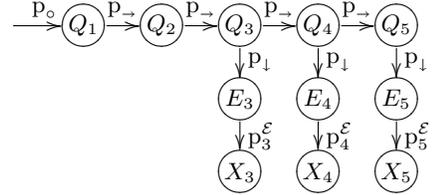
\begin{figure}
\centering
\subfloat[Freezing: EHMM $\A^\fr_{3:5}$\label{fig:freezing}]{
\footnotesize
$\xymatrix@!0@R=3.0em@C=3.2em{
\lt{\ar[rrr]}{\pinit}&
& 
& 
\rv{\Qs[3]} \lt{\ar[r]}{\ptf} \lt{\ar[d]}{\ppf} & 
\rv{\Qs[4]} \lt{\ar[r]}{\ptf} \lt{\ar[d]}{\ppf} &
\rv{\Qs[5]}                   \lt{\ar[d]}{\ppf} 
\\
&
& 
& 
\rv{\Xis[3]} \lt{\ar[d]}{\ExPreds_3} & 
\rv{\Xis[4]} \lt{\ar[d]}{\ExPreds_4} & 
\rv{\Xis[5]} \lt{\ar[d]}{\ExPreds_5} 
\\
&
& 
& 
\rv{\Xs[3]} & 
\rv{\Xs[4]} & 
\rv{\Xs[5]} 
}$}%
\\
\subfloat[Sleeping: EHMM $\A^\sl_{3:5}$\label{fig:sleeping}]{
\footnotesize
$\xymatrix@!0@R=3.0em@C=3.2em{
\lt{\ar[r]}{\pinit} &
\rv{\Qs[1]} \lt{\ar[r]}{\ptf}  & 
\rv{\Qs[2]} \lt{\ar[r]}{\ptf} & 
\rv{\Qs[3]} \lt{\ar[r]}{\ptf} \lt{\ar[d]}{\ppf} &
\rv{\Qs[4]} \lt{\ar[r]}{\ptf} \lt{\ar[d]}{\ppf} &
\rv{\Qs[5]}                   \lt{\ar[d]}{\ppf} 
\\
&
& 
& 
\rv{\Xis[3]} \lt{\ar[d]}{\ExPreds_3} & 
\rv{\Xis[4]} \lt{\ar[d]}{\ExPreds_4} & 
\rv{\Xis[5]} \lt{\ar[d]}{\ExPreds_5} 
\\
&
& 
& 
\rv{\Xs[3]} & 
\rv{\Xs[4]} &
\rv{\Xs[5]} 
}$}%
\caption{Freezing and Sleeping EHMM $\A$ on Example Segment $\xs[3:5]$}\label{fig:sleeping.n.freezing.EHMMs}
\end{figure}

\subsection{Main Result: Construction of the Freezing and Sleeping EHMMs}\label{sec:ll-tbe}

We now present the construction of EHMMs for the freezing and sleeping
algorithms $\fs^\fr$ and $\fs^\sl$. Let $\H$ be a set of learning
experts, each expert $\h \in \H$ presented as an EHMM on basic experts
$\E$. Let $w$ be a prior on $\H$, and let $\alpha$ be a switching rate.
We proceed in two steps. First construct the Bayesian EHMM $\B =
\bayes[w, \H]$ as in \exaref{exa:bayes.on.ehmms}. Recall that $\B$
learns which of the EHMMs in $\H$ predicts best. Second, construct the
freezing EHMM $\fs^\fr[\alpha, \B]$ or the sleeping
EHMM\footnote{Strictly speaking, the Bayesian network in
\figref{fig:tracking.learning.expert.sleeping} is not an EHMM, since the
transition function depends on the time. Nevertheless, this
time-dependency can be removed without any computational overhead using
a process called \emph{unfolding}, see~\citetext{Koolen and De Rooij,\,
\citeyear{us:ceae:corr}}.}
$\fs^\sl[\alpha,\B]$ as shown in
\figref{fig:tracking.learning.expert.ehmms}. Note how, on a switch, both
EHMMs reset the entire state of $\B$, which includes the states of
experts in $\H$. In contrast, $\fs$ only resets its weighting on $\H$,
but does not touch the internal state of the experts in $\H$.

\begin{figure}
\centering
\setlength{\abovedisplayskip}{.25em}
\setlength{\belowdisplayskip}{1em}
\subfloat[{EHMM $\fs^\fr[\alpha, \B]$}\label{fig:tracking.learning.expert.freezing}]{
\begin{minipage}{.46\textwidth}
\[\xymatrix@!0@C=5em@R=3em{
\lt{\ar[r]}{\pinit[\B]} & \rv{\Qs[1]} \lt{\ar[r]}{\ptf} \lt{\ar[d]}{\ppf[\B]} & \rv{\Qs[2]} \lt{\ar[r]}{\ptf} \lt{\ar[d]}{\ppf[\B]} & \rv{\Qs[3]} \lt{\ar[d]}{\ppf[\B]}
\\
& \rv{\Es[1]} \lt{\ar[d]}{\ExPreds_1} & \rv{\Es[2]} \lt{\ar[d]}{\ExPreds_2} & \rv{\Es[3]} \lt{\ar[d]}{\ExPreds_3}
\\
& \rv{\Xs[1]} & \rv{\Xs[2]} & \rv{\Xs[3]}
}\]
\[\ptf ~\df~ (1-\alpha)\ptf[\B]{} + \alpha
\pinit[\B]\onesvec^\transposed\]
{\small Any switch reverts to $\pinit[\B]$, the initial distribution of
$\B$.}
\end{minipage}
}
\\
\subfloat[{EHMM $\fs^\sl[\alpha, \B]$}\label{fig:tracking.learning.expert.sleeping}]{
\begin{minipage}{.46\textwidth}
\[\xymatrix@!0@C=5em@R=3em{
\lt{\ar[r]}{\pinit[\B]} & \rv{\Qs[1]} \lt{\ar[r]}{{\ptf}_{(1)}} \lt{\ar[d]}{\ppf[\B]} & \rv{\Qs[2]} \lt{\ar[r]}{{\ptf}_{(2)}} \lt{\ar[d]}{\ppf[\B]} & \rv{\Qs[3]} \lt{\ar[d]}{\ppf[\B]}
\\
& \rv{\Es[1]} \lt{\ar[d]}{\ExPreds_1} & \rv{\Es[2]} \lt{\ar[d]}{\ExPreds_2} & \rv{\Es[3]} \lt{\ar[d]}{\ExPreds_3}
\\
& \rv{\Xs[1]} & \rv{\Xs[2]} & \rv{\Xs[3]}
}\]
\[{\ptf}_{(t)} ~\df~ (1-\alpha)\ptf[\B]{} + \alpha \del{\ptf[\B]}^t \pinit[\B]\onesvec^\transposed\]
{\small The switch between time $t$ and $t+1$ reverts to $(\ptf[\B])^t
\pinit[\B]$, the $t^\text{th}$ evolution of the initial distribution of
$\B$.}
\end{minipage}
}
\caption{EHMMs for Tracking the EHMM $\B$ with Switching Rate $\alpha$}\label{fig:tracking.learning.expert.ehmms}
\end{figure}

\subsection{Prediction Algorithms}\label{sec:algorithms}
To sequentially predict data using our prediction strategies $\fs^\fr$ and $\fs^\sl$, one needs to run the forward algorithm on their respective EHMMs. An explicit rendering of this process is included in \algoref{alg:forward}.

\begin{algorithm}
\caption{Explicit Forward Algorithm on $\fs^\fos$ for both Freezing and
Sleeping $(\fos \in \{\fr,\sl\})$}\label{alg:forward}
\setlength{\abovedisplayskip}{.2em}
\setlength{\belowdisplayskip}{0cm}
\begin{algorithmic}[1]
\State Construct $\B = \bayes[w, \H]$ with $\Q$, $\pinit$,$\ppf$ and $\ptf$ as in \exaref{exa:bayes.on.ehmms}.
\State Initialisation: $\PredPost \gets \pinit$
\For {$t=1,\ldots$} \Comment{{\tiny Invariant: $\lambda(\q) = \fs^\fos[\alpha, \B](\Qs[t] = \q | \xs[<t])$}}
\State Receive expert advice $\ExPreds_t$.
\State Predict $\xrv_t$ using 
\[\PredPost(\xrv_t) = \sum_{e \in \E, \q \in \Q} \PredPost(\q) \ptf[\q](\e) \ExPred{\e}_t(\xrv_t).\]
\State Observe $\Xs[t] = \xs[t]$. Suffer loss $\loss\big(\lambda(\Xs[t]), \xs[t]\big)$.
\State Loss update: $\PredPost(\q) \gets \PredPost(\q, x_t)/\PredPost(\x_t)$, where 
\[\PredPost(\q, x_t) = \sum_{e \in \E} \PredPost(\q) \ptf[\q](\e) \ExPred{\e}_t(\x_t).\]
\State State evolution:
\[\PredPost \gets
\begin{cases}
(1-\alpha) \ptf \PredPost + \alpha \pinit  & \text{(Freezing)}
\\
(1-\alpha) \ptf \PredPost + \alpha (\ptf)^t \pinit & \text{(Sleeping)}
\end{cases}
\]
\EndFor
\end{algorithmic}
\end{algorithm}

At any time $t$, the algorithm for $\fs^\sl$ only maintains non-zero
weights on hidden states of the input learning experts that are
reachable in \emph{exactly} $t$ steps from the starting states, just
like the original $\fs$ algorithm. It therefore has the same running
time. The algorithm for $\fs^\fr$, however, has to keep track of all
states reachable in \emph{at most} $t$ steps. Consequently, in the worst
case (over input EHMMs) it may be as slow as restarting expert copies (see
\secref{sec:fixedsharelearningexperts}). But if the input EHMMs have a
finite number of hidden states, then its running time is of the same
order as that of $\fs$. And if the states (of the input EHMMs) that are
reachable in exactly $t$ steps are the same ones as the states reachable
in at most $t$ steps, which holds e.g.\ for the drifting-mean expert
$\dm[\theta]$ from the introduction, then we also recover the efficiency of
$\fs$.

\subsection{Loss Bound}\label{sec:lossbound}
\thmref{thm:fixed.share.loss.bound} bounds the regret of $\fs$ compared to the S-TBE reference scheme by a ``switching'' and an ``expert selection'' term. We bound the regret of $\fs^\fr$ and $\fs^\sl$ compared to their LL-TBE reference scheme by the same two terms.

\begin{theorem}
  Fix a set of EHMMs $\H$ on basic experts $\E$, a prior $w$ on $\H$,
  a switching rate $\alpha$ and $\fos \in \set{\fr,\sl}$. Let $\B =
  \bayes[w, \H]$. Then for any data $\xs$, expert predictions
  $\ExPreds_{1:T}$, reference segmentation $\Part$ and assignment of
  experts to segments $\family{\h_\Cell \in \H}{\Cell \in \Part}$
  \newcommand{\braceforimpact}[2]{\underbrace{\vphantom{\sum_{\Cell \in \Part}}#1}_\textnormal{\kern -1cm #2 \kern -1cm}}
\begin{multline*}
  \loss\big(\fs^\fos\sbr{\alpha, \B}, \xs\big) ~\le~ 
\\
    \braceforimpact{\sum_{\Cell \in \Part} \loss(\h_\Cell^\fos, \xs[\Cell])}{LL-TBE ref.\ scheme}
      + \braceforimpact{(T-1) \ent(\alpha^*, \alpha)}{Switching}
      + \braceforimpact{\sum_{\Cell \in \Part} -\log  w(\h_\Cell)}{Expert selection},
\end{multline*}
where $\ent(\alpha^*, \alpha)$ and $\alpha^* = \frac{\card{\Part}-1}{T-1}$ are as in \thmref{thm:fixed.share.loss.bound}.
\end{theorem}

\begin{proof}
The proof proceeds like that of \thmref{thm:fixed.share.loss.bound}.
  Lower-bounding transitions between segments by $\alpha
  \pinit[\B]\onesvec^\transposed$ (freezing) or $\alpha
  \del{\ptf[\B]}^t \pinit[\B]\onesvec^\transposed$ (sleeping), and
  transitions within each segment by $(1-\alpha)\ptf[\B]{} $, we get
\begin{equation}
  \fs^\fos[\alpha, \B]
    \geq \alpha^{\card{\Part}-1}(1-\alpha)^{T-\card{\Part}}
                \prod_{\Cell \in \Part} \B^\fos_\Cell(\xs[\Cell]),
\end{equation}
where $\B^\fos_\Cell$ denotes the result of freezing or sleeping $\B$ on
segment $\Cell \in \Part$ as in \figref{fig:sleeping.n.freezing.EHMMs}.
Observe that freezing and sleeping distribute over taking the Bayesian
mixture: $\B^\fos_\Cell = 
\bayes[w, \H^\fos_\Cell]$, where $\H^\fos_\Cell \df \set{\h^\fos_\Cell
\mid \h \in \H}$. 
As $\bayes[w,\H^\fos_\Cell](\xs[\Cell]) = \sum_{\h} w(\h)
\h^\fos_\Cell(\xs[\Cell]) \ge w(\h_\Cell) \h^\fos_\Cell(\xs[\Cell])$,
the theorem follows from 
\eqref{eqn:introchainrule}, like in the proof of
\thmref{thm:fixed.share.loss.bound}.
\end{proof}

\section{Conclusion}\label{sec:discussion}
We revisited the tracking the best expert reference scheme (TBE), which
asks for a strategy for prediction with expert advice that suffers small
additional loss compared to the best expert per segment. This goal is
natural when the characteristics of the data, and hence the best expert,
are different between segments.

For learning experts, the standard interpretation of experts as black
boxes implies training the experts on all data. We proposed a variation,
adapted to learning experts, in which experts are only trained on the
segment on which they are evaluated. Our scheme is able to exploit
patterns in the data \emph{per segment}, leading to smaller loss.

Although in general extending the standard fixed-share algorithm to our
setting will slow it down by a factor of $T$ on $T$ outcomes, we showed
that no such slowdown is necessary if the learning experts can be
represented as expert hidden Markov models (EHMMs).
We proved the loss bounds one would expect based on the
loss bound for the original fixed-share algorithm.

\subsection{Discussion and Future Work}
\paragraph{Learning the Switching Rate}
Like fixed share, our algorithms depend on a switching rate parameter
$\alpha$, which has to be fixed. Instead, one may want to tune $\alpha$
automatically based on the data. For $\fs$ this can be done efficiently
(see~\citetext{De Rooij and Van Erven,\, \citeyear{DeRooijVanErven2009}}
and references therein). The same methods transfer directly to $\fs^\fr$
and $\fs^\sl$.

\paragraph{S-TBE vs LL-TBE}

We have discussed experts that learn only on their assigned segment.
Perhaps surprisingly, this does \emph{not always} increase performance.
For example, we may have homogeneous data and experts that learn its
global pattern at different rates. In such cases we clearly want to
train each expert on all observations and, by switching at the right
times, select the expert that has learned most until then. This
scenario is analysed by~\Citet{threemusketeers07}, where experts are
parameter estimators for a series of statistical models of increasing
complexity.

\paragraph{Partitions instead of Segmentations}

Rather than split the data into segments as in the TBE reference scheme,
one may wish to partition it arbitrarily into cells such that
observations in the same cell need not be consecutive. Like fixed share,
the corresponding algorithm \citep{bousquet2002} can be generalised to
the LL-TBE setting without increasing its running time. In this case
naively introducing copies of the experts for all possible partitions is
infeasible: it would slow down the algorithm by an exponential factor
$2^T$ on $T$ outcomes.

\normalsize

{\small\baselineskip=9pt\bibliography{writeup,experts}}

\normalsize

\end{document}